\newif\ifCONF
\CONFfalse
\newif\ifarXiv
\arXivfalse
\newif\ifWP
\WPfalse

\CONFtrue		

\newif\ifnotCONF	
\notCONFtrue
\ifCONF\notCONFfalse\fi

\newif\ifnotarXiv	
\notarXivtrue
\ifarXiv\notarXivfalse\fi

\newif\ifTR		
\TRfalse
\ifarXiv\TRtrue\fi
\ifWP\TRtrue\fi
\newif\ifnotTR
\notTRtrue
\ifarXiv\notTRfalse\fi
\ifWP\notTRfalse\fi

\ifCONF
\documentclass{article}
\fi

\ifarXiv
\documentclass{article}
\fi

\ifWP
\documentclass{article}
\usepackage{amsmath,amsthm,amsfonts,amssymb,latexsym,graphicx,stmaryrd}
  \input{OT2enc.def}
  
  \usepackage{CJK}
\input{/Doc/Computing/Latex/kp.txt}
\fi

\usepackage{graphicx} 
\usepackage{subfigure}

\usepackage{natbib}

\usepackage{algorithm}
\usepackage{algorithmic}
\usepackage{amsthm}
\newtheorem{theorem}{Theorem}
\newtheorem{lemma}{Lemma}
\theoremstyle{definition}
\newtheorem*{remark}{Remark}
\usepackage{amsmath}

\usepackage{hyperref}
\usepackage{microtype}

\ifnotCONF
\fi

\ifCONF
\usepackage[accepted]{icml2012}

\icmltitlerunning{Plug-in martingales}
\fi

\ifarXiv
\title{Plug-in martingales for testing exchangeability on-line}
\author{Valentina Fedorova, Alex Gammerman,\\ Ilia Nouretdinov and Vladimir Vovk\\
  Computer Learning Research Centre\\
  Royal Holloway, University of London, UK\\
  \texttt{\{valentina,ilia,alex,vovk\}@cs.rhul.ac.uk}}
\fi

\ifWP
\title{Plug-in martingales for testing exchangeability on-line}
\author{Valentina Fedorova, Alex Gammerman,\\ Ilia Nouretdinov, and Vladimir Vovk\\[3mm]
  Computer Learning Research Centre\\
  Royal Holloway, University of London, UK\\[3mm]
  \texttt{\{valentina,ilia,alex,vovk\}@cs.rhul.ac.uk}}

\fi

\allowdisplaybreaks[4]

\begin{document}
\ifCONF
\twocolumn[
\icmltitle{Plug-in martingales for testing exchangeability on-line}

\icmlauthor{Valentina Fedorova}{valentina@cs.rhul.ac.uk}
\icmlauthor{Alex Gammerman}{alex@cs.rhul.ac.uk}
\icmlauthor{Ilia Nouretdinov}{ilia@cs.rhul.ac.uk}
\icmlauthor{Vladimir Vovk}{v.vovk@rhul.ac.uk}
\icmladdress{Computer Learning Research Centre, Royal Holloway,
University of London, Egham, Surrey, TW20 0EX, UK}

\icmlkeywords{martingales, on-line testing of assumptions, conformal predictors}

\vskip 0.3in
]
\fi

\ifnotCONF
  \maketitle
\fi

\begin{abstract}
A standard assumption in machine learning is the exchangeability of data,
which is equivalent to assuming that the examples are generated from the same
probability distribution independently. This paper is devoted to testing
the assumption of exchangeability on-line: the examples arrive one by one,
and after receiving each example we would like to have
a valid measure of the degree
to which the assumption of exchangeability has been falsified.
Such measures are provided by exchangeability martingales.
We extend known techniques for constructing exchangeability martingales
and show that our new method is competitive with the martingales introduced before.
Finally we investigate the performance of our
testing method on two benchmark datasets,
USPS and Statlog Satellite data;
for the former, the known techniques give satisfactory results,
but for the latter our new more flexible method becomes necessary. \end{abstract}

\section{Introduction}\label{sec:intro}

Many machine learning algorithms have been developed
to deal with real-life high dimensional data.
In order to state and prove properties of such algorithms it is
standard to assume that the data satisfy the exchangeability assumption
(although some algorithms make different assumptions or,
in the case of prediction with expert advice,
do not make any statistical assumptions at all).
These properties can be violated if the assumption is not satisfied,
which makes it important to test the data for satisfying it.

Note that the popular assumption that the data is i.i.d.\ (independent and identically distributed)
has the same meaning for testing as the exchangeability assumption.
A joint distribution of an infinite sequence of examples is exchangeable
if it is invariant w.r.\ to any permutation of examples.
Hence if the data is i.i.d., its distribution is exchangeable.
On the other hand, by de Finetti's theorem \citep[see, e.g.,][p.~28]{schervish:1995}
any exchangeable distribution on the data
(a potentially infinite sequence of examples)
is a mixture of distributions under which the data is i.i.d.
Therefore, testing for exchangeability is equivalent to testing for being i.i.d.

Traditional statistical approaches to testing are inappropriate
for high dimensional data \citep[see, e.g.,][pp.~6--7]{vapnik:1998}.
To address this challenge a previous study \citep{vovk/nouretdinov/gammerman:2003}
suggested a way of on-line testing by employing the theory of conformal prediction
and calculating exchangeability martingales.
Basically testing proceeds in two steps.
The first step is implemented by a conformal predictor that outputs a sequence of p-values.
The sequence is generated in the on-line mode:
examples are presented one by one and for each new example a p-value is calculated
from this and all the previous examples.
For the second step the authors introduced exchangeability martingales that are functions of the p-values
and track the deviation from the assumption.
Once the martingale grows up to a large value
(20 and 100 are convenient rules of thumb)
the exchangeability assumption can be rejected for the data.

This paper proposes a new way of constructing martingales in the second step of testing.
To construct an exchangeability martingale based on the sequence of p-values we need a betting function,
which determines the contribution of a p-value to the value of the martingale.
In contrast to the previous studies that use a fixed betting function
the new martingale tunes its betting function to the sequence
to detect any deviation from the assumption.
We show that this martingale,
which we call a plug-in martingale,
is competitive with all the martingales covered by the previous studies;
namely, asymptotically the former grows faster than the latter.

\subsection{Related work}

The first procedure of testing exchangeability on-line is described
in \citet{vovk/nouretdinov/gammerman:2003}.
The core testing mechanism is an exchangeability martingale.
Exchangeability martingales are constructed using a sequence of p-values.
The algorithm for generating p-values assigns small p-values to unusual examples.
It implies the idea of designing martingales that would have a large value if too many small p-values were generated,
and suggests corresponding power martingales.
Other martingales (simple mixture and sleepy jumper) implement more
complicated strategies, but follow the same idea of scoring on small p-values.

\citet{ho:2005} applies power martingales to the problem of change detection
in time-varying data streams.
The author shows that small p-values inflate the martingale values
and suggests to use the martingale difference as another test for the problem.

\subsection{This paper}

To the best of our knowledge, no study has aimed to find any other ways of translating p-values
into a martingale value.
In this paper we propose a new more flexible method of constructing exchangeability martingales
for a given sequence of p-values.

The rest of the paper is organised as follows.
Section \ref{sec:definition} gives the definition of exchangeability martingales.
Section \ref{sec:construction} presents the construction of plug-in exchangeability martingales,
explains the rationale behind them, and compares them to the power martingales,
which have been used previously.
Section \ref{sec:results} shows experimental results of testing two real-life datasets for exchangeability;
for one of these datasets power martingales work satisfactorily and for the other one the greater flexibility
of plug-in martingales becomes essential.
Section \ref{sec:conclusion} summarises the paper.

\section{Exchangeability martingales}
\label{sec:definition}

This section outlines necessary definitions and results of the previous studies.

\subsection{Exchangeability}

Consider a sequence of random variables $\bigl(Z_1,Z_2,\ldots )$ that all take values
in the same example space.
Then the joint probability distribution $\textbf{P}(Z_1,\ldots,Z_N)$
of a finite number of the random variables is \emph{exchangeable}
if it is invariant under any permutation of the random variables.
The joint distribution of infinite number of random variables $\bigl(Z_1,Z_2,\ldots )$
is \emph{exchangeable} if the marginal distribution $\textbf{P}(Z_1,\ldots,Z_N)$
is exchangeable for every $N$.

\subsection{Martingales for testing}

As in \citet{vovk/nouretdinov/gammerman:2003},
the main tool for testing exchangeability on-line is a martingale.
The value of the martingale reflects the strength of evidence against
the exchangeability assumption. An \emph{exchangeability martingale}
is a sequence of non-negative random variables $S_0,S_1,\ldots$
that keep the conditional expectation:
\begin{align*}
  S_n &\ge 0\\
  S_{n} &= \textbf{E}(S_{n+1}\mid S_1,\ldots,S_{n}),
\end{align*}
where \textbf{E} refers to the expected value
with respect to any exchangeable distribution on examples.
We also assume $S_0 = 1$.
Note that we will obtain an equivalent definition
if we replace ``any exchangeable distribution on examples''
by ``any distribution under which the examples are i.i.d.''\
(remember the discussion of de Finetti's theorem in Section~\ref{sec:intro}).

To understand the idea behind martingale testing we can imagine a game where a player
starts from the capital of 1, places bets on the outcomes of a sequence of events,
and never risks bankruptcy. Then a martingale corresponds to a strategy of the player,
and its value reflects the acquired capital.
According to Ville's inequality \citep[see][p.~100]{ville:1939},
$$\textbf{P}\Bigl\{\exists n: S_n\geq C\Bigr\}\leq 1/C, \quad \forall C\ge1,$$
it is unlikely for any $S_n$ to have a large value.
For the problem of testing exchangeability, if the final value of a martingale is large
then the exchangeability assumption for the data can be rejected with the corresponding probability.

\subsection{On-line calculation of p-values}

Let $(z_1, z_2, \ldots)$ denote a sequence of examples. Each example $z_i$ is the vector representing
a set of attributes $x_i$ and a label $y_i$: $z_i = (x_i,y_i)$.
In this paper we use conformal predictors to generate a sequence of p-values
that corresponds to the given examples. The general idea of conformal prediction is
to test how well a new example fits to the previously observed examples.
For this purpose a ``nonconformity measure'' is defined.
This is a function that estimates the strangeness of one example
with respect to others:
$$
  \alpha_i = A\Bigl(z_i,\{z_1,\ldots,z_n\}\Bigr),
$$
where in general $\{\ldots\}$ stands for a multiset
(the same element may be repeated more than once) rather than a set.
Typically,
each example is assigned a ``nonconformity score'' $\alpha_i$
based on some prediction method.
In this paper we deal with the classification problem and the 1-Nearest Neighbor (1-NN) algorithm is used
as the underling method to compute the nonconformity scores.
The algorithm is simple but it works well enough in many cases
\citep[see, e.g.,][pp.~422--427]{hastie:2001}.
A natural way to define the nonconformity score of an example
is by comparing its distance to the examples with the same label
to its distance to the examples with a different label:
\begin{equation}\label{eq:NN}
  \alpha_i = \frac{\min_{j \neq i: y_i = y_j}d(x_i,x_j)}{\min_{j \neq i: y_i \neq y_j}d(x_i,x_j)},
\end{equation}
where $d(x_i,x_j)$ is the Euclidean distance.
According to the chosen nonconformity measure,
$\alpha_i$ is high if the example is close to
another example with a different label and far from any examples with the same label.

Using the calculated nonconformity scores of all observed examples,
the p-value $p_n$ that corresponds to an example $z_n$ is calculated as
$$p_n = \frac{\#\{i: \alpha_i > \alpha_n\} + \theta_n \#\{i: \alpha_i = \alpha_n \}}{n},$$
where $\theta_n$ is a random number from $[0,1]$
and the symbol $\#$ means the cardinality of a set.
Algorithm~\ref{alg:generate_pVal} summarises the process of on-line calculation of p-values
(it is clear that it can also be applied to a finite dataset $(z_1,\ldots,z_n)$
producing a finite sequence $(p_1,\ldots,p_n)$ of p-values).

\begin{algorithm}[tb]
\caption{Generating p-values on-line}
\label{alg:generate_pVal}
\begin{algorithmic}
   \STATE {\bfseries Input:} $(z_1, z_2,\ldots)$ data for testing
   \STATE {\bfseries Output:} $(p_1, p_2,\ldots)$ sequence of p-values

   \FOR{$i=1,2,\ldots$}
   \STATE observe a new example $z_i$
   \FOR{$j=1$ {\bfseries to} $i$}
   \STATE $\alpha_j = A\Bigl(z_j,\{z_1,\ldots,z_i\}\Bigr)$
   \ENDFOR
   \STATE  $p_i = \frac{\#\{j: \alpha_j > \alpha_i \} + \theta_i \#\{j: \alpha_j = \alpha_i \}}{i}$
   \ENDFOR
\end{algorithmic}
\end{algorithm}

The following is a standard result in the theory of conformal prediction
(see, e.g., \citealt{vovk/nouretdinov/gammerman:2003}, Theorem~1).
\begin{theorem} 
If examples $(z_1, z_2,\ldots)$
(resp.\ $(z_1, z_2,\ldots,$ $z_n)$)
satisfy the exchangeability assumption,
Algorithm \ref{alg:generate_pVal} produces p-values
$(p_1, p_2,\ldots)$ (resp.\ $(p_1, p_2,\ldots, p_n)$)
that are independent and uniformly distributed in $[0,1]$.
\end{theorem}
The property that the examples generated by an exchangeable distribution
provide uniformly and independently
distributed p-values allows us to test exchangeability
by calculating martingales as functions of the p-values.

\section{Martingales based on p-values} \label{sec:construction}

This section focuses on the second part of testing:
given the sequence of p-values a martingale is
calculated as a function of the p-values.

For each $i\in\{1,2,\ldots\}$, let $f_i:[0,1]^i\to[0,\infty)$.
Let $(p_1,p_2,\ldots)$ be the sequence of p-values generated by Algorithm \ref{alg:generate_pVal}.
We consider martingales $S_n$ of the form
\begin{equation}\label{eq:mart}
  S_n = \prod_{i=1}^n f_i(p_i),
  \quad
  n=1,2,\ldots,
\end{equation}
where we denote $f_i(p) = f_i(p_1,\ldots,p_{i-1}, p)$
and call the function $f_i(p)$ a \emph{betting function}.

To be sure that (\ref{eq:mart}) is indeed a martingale
we need the following constraint on the betting functions $f_i$:
$$\displaystyle \int_0^1 f_i(p)d p = 1, \quad i = 1,2,\ldots$$
Then we can check:
\begin{multline*}
  \textbf{E}(S_{n+1}\mid S_0,\ldots,S_n)
  =
  \int_0^1 \prod_{i=1}^n \Bigl(f_i(p_i) \Bigr)f_{n+1}(p) d p\\
  =
  \prod_{i=1}^n \Bigl(f_i(p_i)\Bigr) \int_0^1 f_{n+1}(p) dp
  =
  \prod_{i=1}^n f_i(p_i)
  =
  S_n.
\end{multline*}

Using representation (\ref{eq:mart}) we can update the  martingale on-line:
having calculated a p-value $p_i$ for a new example in Algorithm \ref{alg:generate_pVal}
the current martingale value becomes $S_i = S_{i-1}\cdot f_i(p_i)$.
To define the martingales completely we need to describe the betting functions $f_i$.

\subsection{Previous results: power and simple mixture martingales}

Previous studies \citep{vovk/nouretdinov/gammerman:2003} have proposed to use a fixed betting function
$$\forall i:~f_i(p) = \varepsilon p^{\varepsilon-1}, $$
where $\varepsilon \in [0,1]$.
Several martingales were constructed using the function.
The \emph{power martingale} for some $\varepsilon$, denoted as ${M}^\varepsilon_n$,
is defined as
$$M^\varepsilon_n =\prod_{i=1}^n \varepsilon p_i^{\varepsilon-1}.$$
The \emph{simple mixture} martingale, denoted as ${M}_n$, is the mixture of power martingales
over different $\varepsilon \in [0,1]$:
$${M}_n = \int_0^1 {M}^\varepsilon_n d\varepsilon.$$

\begin{figure}[tb]
\vskip 0.2in
\begin{center}
\centerline{\includegraphics[width=0.7\columnwidth]{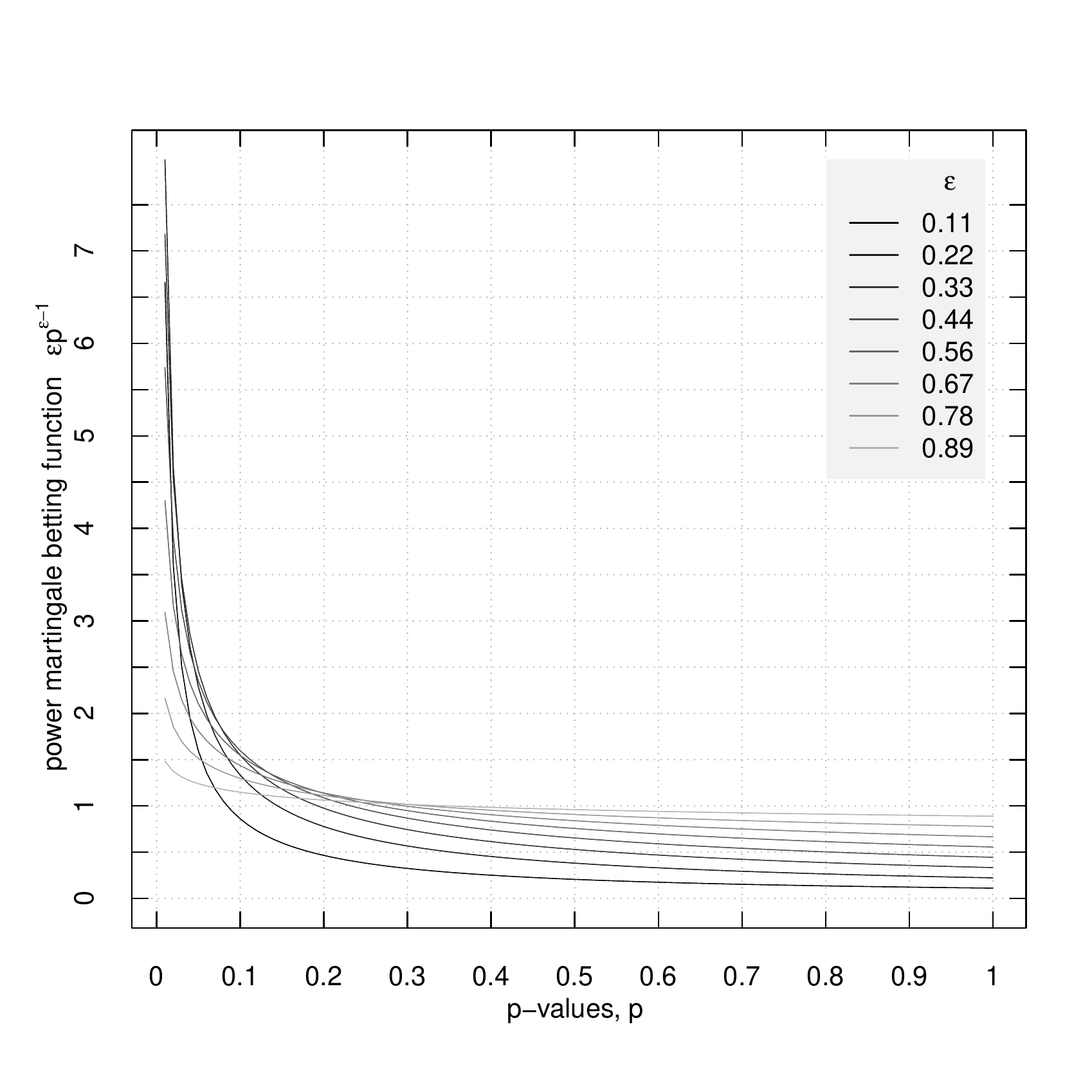}}
\caption{The betting functions that are used to construct the power and simple mixture martingales.}
\label{fig:power_betting_f}
\end{center}
\vskip -0.2in
\end{figure}

Such a martingale will grow only if there are many small p-values in the sequence.
This follows from the shape of the betting functions: see Figure \ref{fig:power_betting_f}.
If the generated p-values concentrate in any other part of the unit interval,
we cannot expect the martingale to grow.
So it might be difficult to reject the assumption of exchangeability for such sequences.

\subsection{New plug-in approach}

\subsubsection{Plug-in martingale}
Let us use an estimated probability density function
as the betting function $f_i(p)$.
At each step the probability density function is estimated
using the accumulated p-values:
\begin{equation}\label{eq:rho}
  \rho_i(p) = \widehat{\rho}(p_1,\ldots, p_{i-1},p),
\end{equation}
where $\widehat{\rho}(p_1,\ldots, p_{i-1},p)$
is the estimate of the probability density function
using the p-values $p_1,\ldots, p_{i-1}$ output by Algorithm \ref{alg:generate_pVal}.

Substituting these betting functions into (\ref{eq:mart}) we get a new martingale
that we call a \emph{plug-in} martingale.
The martingale avoids betting if the p-values are distributed uniformly,
but if there is any peak it will be used for betting.

\paragraph{Estimating a probability density function.}
 In our experiments we have used the statistical environment and language R.
 The \texttt{density} function in its \texttt{Stats} package
 implements kernel density estimation with different parameters.
 But since p-values always lie in the unit interval,
 the standard methods of kernel density estimation lead to poor results for the points that are near the boundary.
 To get better results for the boundary points
 the sequence of p-values is reflected to the left from zero and to the right from one.
 Then the kernel density estimate is calculated using the extended sample
 $\cup_{i=1}^n\bigl\{-p_i, p_i, 2 - p_i\bigr\}$.
 The estimated density function is set to zero outside the unit interval and then
 normalised to integrate to one.
 For the results presented in this paper the parameters used are
 the Gaussian kernel and Silverman's ``rule of thumb'' for bandwidth selection.
 Other settings have been tried as well, but the results are comparable and lead
 to the same conclusions.

The values $S_n$ of the plug-in martingale can be updated recursively.
Suppose computing the nonconformity scores $(\alpha_1,\ldots,\alpha_n)$
from $(z_1,\ldots,z_n)$ takes time $g(n)$
and evaluating (\ref{eq:rho}) takes time $h(n)$.
Then updating $S_{n-1}$ to $S_n$ takes time
$O(g(n)+n+h(n))$:
indeed, it is easy to see that calculating the rank of $\alpha_n$ in the multiset $\{\alpha_1,\ldots,\alpha_n\}$
takes time $\Theta(n)$.

The performance of the plug-in martingale on real-life datasets will be presented
in Section~\ref{sec:results}.
The rest of the current section proves that the plug-in martingale
provides asymptotically a better growth rate than any martingale with a fixed betting
function. To prove this asymptotical property of the plug-in martingale we need
the following assumptions.

\subsubsection{Assumptions}

Consider an infinite sequence of p-values $(p_1, p_2, \ldots)$.
(This is simply a deterministic sequence.)
For its finite prefix $(p_1,\ldots, p_n)$ define the corresponding empirical
probability measure $\textbf{P}_n$:
for a Borel set $A$ in \textbf{R},
$$\textbf{P}_n(A) = \frac{\#\{i=1,\ldots,n: p_i \in A\}}{n}.$$
We say that the sequence $(p_1, p_2, \ldots)$ is \emph{stable}
if there exists a probability measure \textbf{P} on $\textbf{R}$ such that:
\begin{enumerate}
  \item $\textbf{P}_n\xrightarrow[n\rightarrow\infty]{\mbox{weak}} \textbf{P}$;
  \item there exists a positive continuous density function $\rho(p)$ for $\textbf{P}$:
   for any Borel set $A$ in \textbf{R},
   $\textbf{P}(A) = \int_A \rho(p)d p$.
\end{enumerate}
Intuitively, the stability means that asymptotically the sequence of p-values
can be described well by a probability distribution.

Consider a sequence $(f_1(p), f_2(p),\ldots)$ of betting functions.
(This is simply a deterministic sequence of functions $f_i:[0,1]\to[0,\infty)$,
although we are particularly interested in the functions $f_i(p)=\rho_i(p)$, as defined in (\ref{eq:rho}).)
We say that this sequence is \emph{consistent for $(p_1,p_2,\ldots)$} if
$$\log\bigl(f_n(p)\bigr) \xrightarrow[n\rightarrow\infty]{\mbox{uniformly in $p$}} \log\bigl(\rho(p)).$$
Intuitively, consistency is an assumption about the algorithm that we use to estimate the function $\rho(p)$;
in the limit we want a good approximation.

\subsubsection{Growth rate of plug-in martingale}

The following result says
that, under our assumptions, the logarithmic growth rate of the plug-in martingale
is better than that of any martingale with a fixed betting function
(remember that by a betting function we mean any function mapping $[0,1]$ to $[0,\infty)$).
\begin{theorem}\label{thm:plug-in_result}
If a sequence $(p_1,p_2,\ldots)\in[0,1]^{\infty}$ is stable and
a sequence of betting functions $\bigl(f_1(p), f_2(p),\ldots\bigr)$ is consistent for it
then, for any positive continuous betting function $f$,
$$
  \liminf_{n\rightarrow\infty}\left(\frac{1}{n} \sum_{i=1}^n\log\bigl(f_i(p_i)\bigr) -
  \frac{1}{n} \sum_{i=1}^n\log\bigl(f(p_i)\bigr)\right) \geq 0
$$
\end{theorem}

First we explain the meaning of Theorem~\ref{thm:plug-in_result} and then prove it.
According to representation (\ref{eq:mart}) after $n$ steps the martingale grows to
\begin{equation}
\label{eq:mart_prod}
\prod_{i=1}^n f_i(p_i).
\end{equation}
Note that if for any p-value $p\in [0,1] $ we have $f_i(p)=0$
then the martingale can become zero and will never change after that.
Therefore, it is reasonable to consider positive $f_i(p)$.
Then we can rewrite product (\ref{eq:mart_prod}) as sum of logarithms,
which gives us the logarithmic growth of the martingale:
$$\sum_{i=1}^n \log \Bigl( f_i(p_i) \Bigr).$$
We assume that the sequence of p-values is stable and the
sequence of estimated probability density functions that is used to construct the plug-in
martingale is consistent.
Then the limit inequality from Theorem~\ref{thm:plug-in_result} states
that the logarithmic growth rate of the plug-in martingale is asymptotically at least as high as that
of any martingale with a fixed betting function (which have been suggested in previous studies).

To prove Theorem~\ref{thm:plug-in_result} we will use the following lemma.
\begin{lemma}\label{thm:lem1}
  For any probability density functions $\rho$ and $f$
  (so that $\int_0^1 \rho(p) dp =1$ and $\int_0^1 f(p)d p = 1$),
  $$
    \int_0^1 \log\Bigl(\rho(p)\Bigr)\rho(p) dp \geq \int_0^1 \log\Bigl(f(p)\Bigr)\rho(p) dp.
  $$
\end{lemma}

\begin{proof}[Proof of Lemma~\ref{thm:lem1}]
It is well known \citep[][p.~14]{kullback:1959} that the Kullback--Leibler divergence
is always non-negative:
$$
  \int_0^1\log\Bigl(\frac{\rho(p)}{f(p)}\Bigr)\rho(p) dp \geq 0.
$$
This is equivalent to the inequality asserted by Lemma~\ref{thm:lem1}.
%
\end{proof}

\begin{proof}[Proof of Theorem~\ref{thm:plug-in_result}]
Suppose that,
contrary to the statement of Theorem~\ref{thm:plug-in_result},
there exists $\delta >0$ such that
\begin{equation}\label{eq:proof_contr1}
  \liminf_{n\rightarrow\infty}
  \left(
    \frac{1}{n} \sum_{i=1}^n\log\bigl(f_i(p_i)\bigr) -
    \frac{1}{n} \sum_{i=1}^n\log\bigl(f(p_i)\bigr)
  \right) < -\delta.
\end{equation}
Then choose an $\varepsilon$ satisfying $0<\varepsilon < \delta/4$.

Substituting the definition of $\rho(p)$ into Lemma~\ref{thm:lem1} we obtain
\begin{equation}\label{eq:proof_step1}
\int_0^1 \log\Bigl(\rho(p)\Bigr)d\textbf{P} \geq \int_0^1 \log\Bigl(f(p)\Bigr)d\textbf{P}.
\end{equation}
From the stability of $(p_1, p_2, \ldots)$ it follows that
there exists a number $N_1 = N_1(\varepsilon)$ such that, for all $n > N_1$,
$$
  \left|
    \int_0^1 \log\Bigl(f(p)\Bigr)d\textbf{P}_n -
    \int_0^1 \log\Bigl(f(p)\Bigr)d\textbf{P}
  \right|
  <
  \varepsilon
$$
and
$$
  \left|
    \int_0^1 \log\Bigl(\rho(p)\Bigr)d\textbf{P}_n -
    \int_0^1 \log\Bigl(\rho(p)\Bigr)d\textbf{P}
  \right|
  <
  \varepsilon.
$$
Then inequality (\ref{eq:proof_step1}) implies that, for all $n \geq N_1$,
$$
\int_0^1 \log\Bigl(\rho(p)\Bigr)d\textbf{P}_n \geq
\int_0^1 \log\Bigl(f(p)\Bigr)d\textbf{P}_n - 2\varepsilon.
$$
By the definition of the probability measure $\textbf{P}_n$,
the last inequality is the same thing as
\begin{equation}\label{eq:proof_step3}
  \frac{1}{n}\sum_{i=1}^n \log\Bigl(\rho(p_i)\Bigr) \geq
  \frac{1}{n}\sum_{i=1}^n \log\Bigl(f(p_i)\Bigr) - 2\varepsilon.
\end{equation}
By the consistency of $\bigl(f_1(p), f_2(p),\ldots\bigr)$
there exists a number $N_2= N_2(\varepsilon)$ such that,
for all $i > N_2$ and all $p\in[0,1]$,
\begin{equation}
\label{eq:proof_step4_c1}
\Bigl| \log \bigl( f_i(p) \bigr) - \log \bigl( \rho(p) \bigr) \Bigr| <
\varepsilon.
\end{equation}
Let us define the number
\begin{equation}
\label{eq:proof_step4_c2}
M = \displaystyle\max_{i,p}
\bigl| \log \bigl( f_i(p) \bigr) - \log \bigl( \rho(p) \bigr)\bigr|.
\end{equation}
From (\ref{eq:proof_step4_c1}) and (\ref{eq:proof_step4_c2}) we have
\begin{equation}
\label{eq:proof_step4_c3}
\bigl| \log \bigl( f_i(p) \bigr) - \log \bigl( \rho(p) \bigr)\bigr| \leq
\left\{
 \begin{array}{ll}
   M,& i \leq N_2\\
   \varepsilon,& i > N_2.
 \end{array}
\right.
\end{equation}

Denote $N_3 = \max(N_1,N_2)$.
Then, using (\ref{eq:proof_step4_c3}) and (\ref{eq:proof_step3}),
we obtain, for all $n > N_3$,
$$
 \frac{1}{n}\sum_{i=1}^n \log\Bigl(f_i(p_i)\Bigr) \geq
\frac{1}{n}\sum_{i=1}^n \log\Bigl(f(p_i)\Bigr) - 3\varepsilon - \frac{M N_3}{n}.
$$
Denoting $N_4 = \max(N_3, \frac{M N_3}{\varepsilon})$,
we can rewrite the last inequality as
$$
\frac{1}{n}\sum_{i=1}^n \log\Bigl(f_i(p_i)\Bigr) \geq
\frac{1}{n}\sum_{i=1}^n \log\Bigl(f(p_i)\Bigr) - 4\varepsilon,
$$
for all $n > N_4$.
Finally, recalling that $\varepsilon < \frac{\delta}{4}$, we have, for all $n > N_4$,
$$
\frac{1}{n}\sum_{i=1}^n \log\Bigl(f_i(p_i)\Bigr) -
\frac{1}{n}\sum_{i=1}^n \log\Bigl(f(p_i)\Bigr) \geq - \delta.
$$
This contradicts (\ref{eq:proof_contr1})
and therefore completes the proof of Theorem~\ref{thm:plug-in_result}.
\end{proof}

\section{Empirical results}\label{sec:results}

In this section we investigate the performance of our plug-in martingale and
compare it with that of the simple mixture martingale. Two real-life datasets have
been tested for exchangeability: the USPS dataset and the Statlog Satellite dataset.

\begin{figure}[tb]
\vskip 0.2in
\begin{center}
\centerline{\includegraphics[width=0.7\columnwidth]{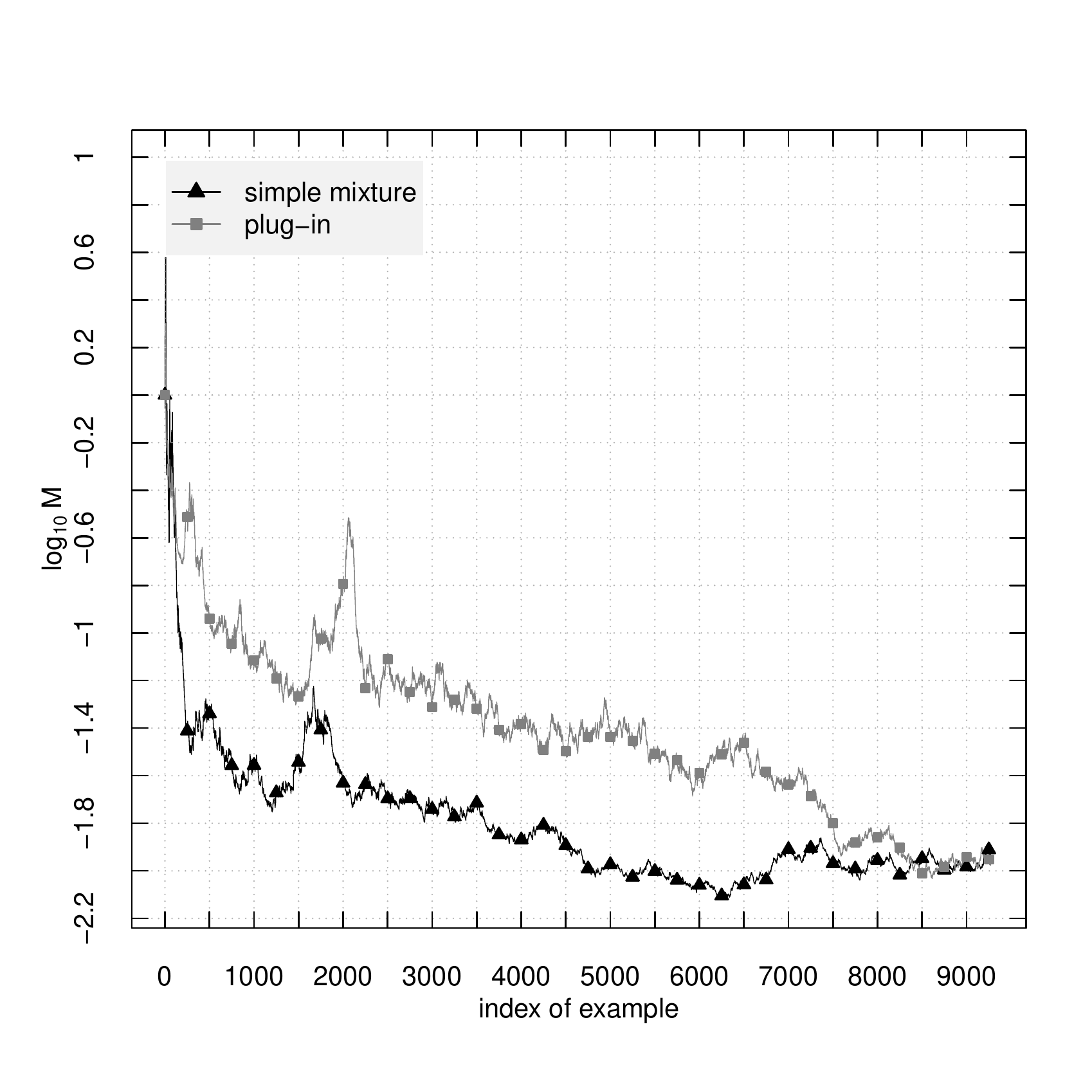}}
\caption{The growth of the martingales for the
USPS dataset randomly shuffled before on-line testing. The exchangeability
assumption is satisfied: the final martingale values are about $0.01$.}
\label{fig:usps_mixed_martingales}
\end{center}
\vskip -0.2in
\end{figure}

\begin{figure}[tb]
\vskip 0.2in
\begin{center}
\centerline{\includegraphics[width=0.7\columnwidth]{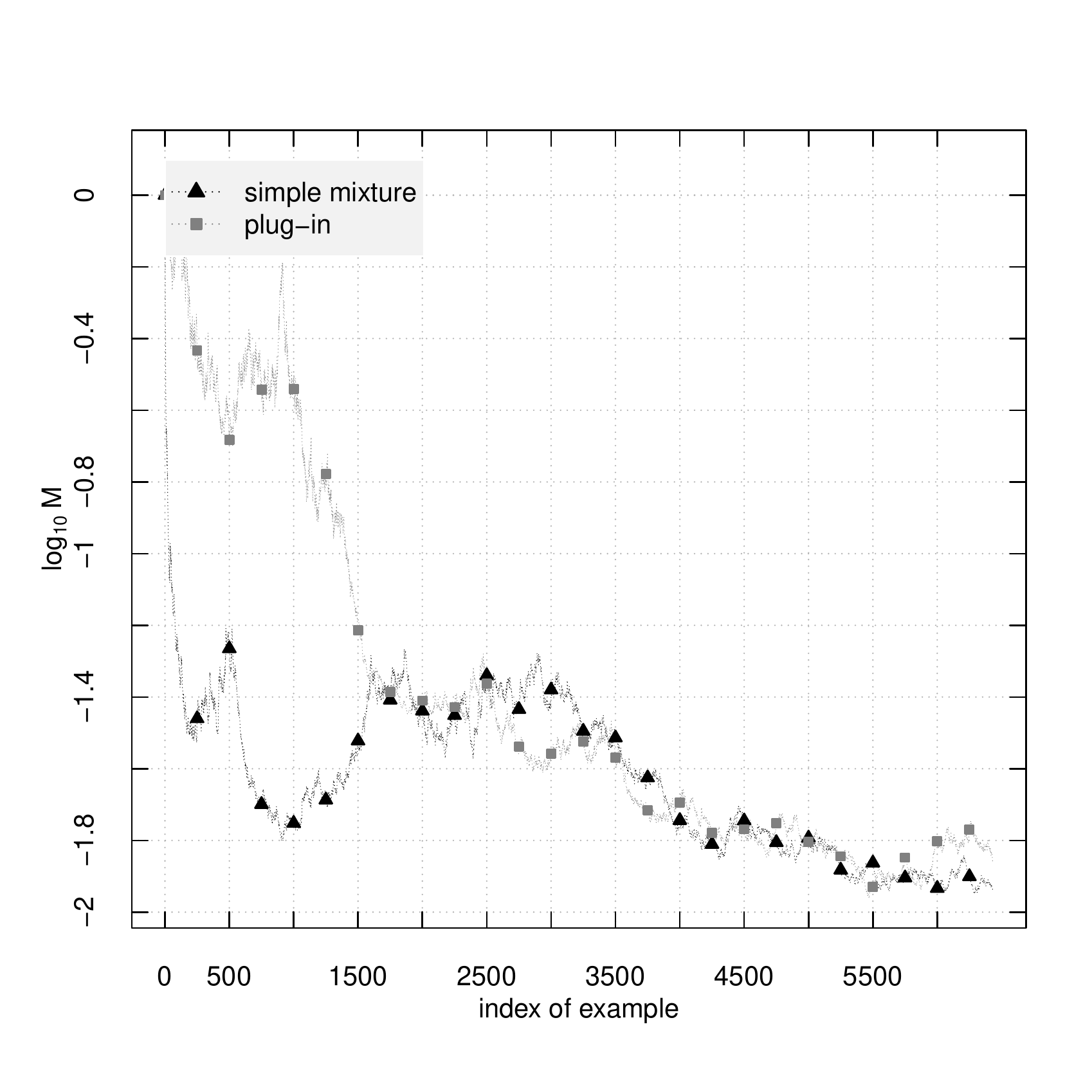}}
\caption{The growth of the martingales for the Statlog Satellite dataset randomly shuffled before on-line testing.
 The exchangeability assumption is satisfied: the final martingale values are about $0.01$.}
\label{fig:statlog_mixed_martingales}
\end{center}
\vskip -0.2in
\end{figure}

\begin{figure}[tb]
\vskip 0.2in
\begin{center}
\centerline{\includegraphics[width=0.7\columnwidth]{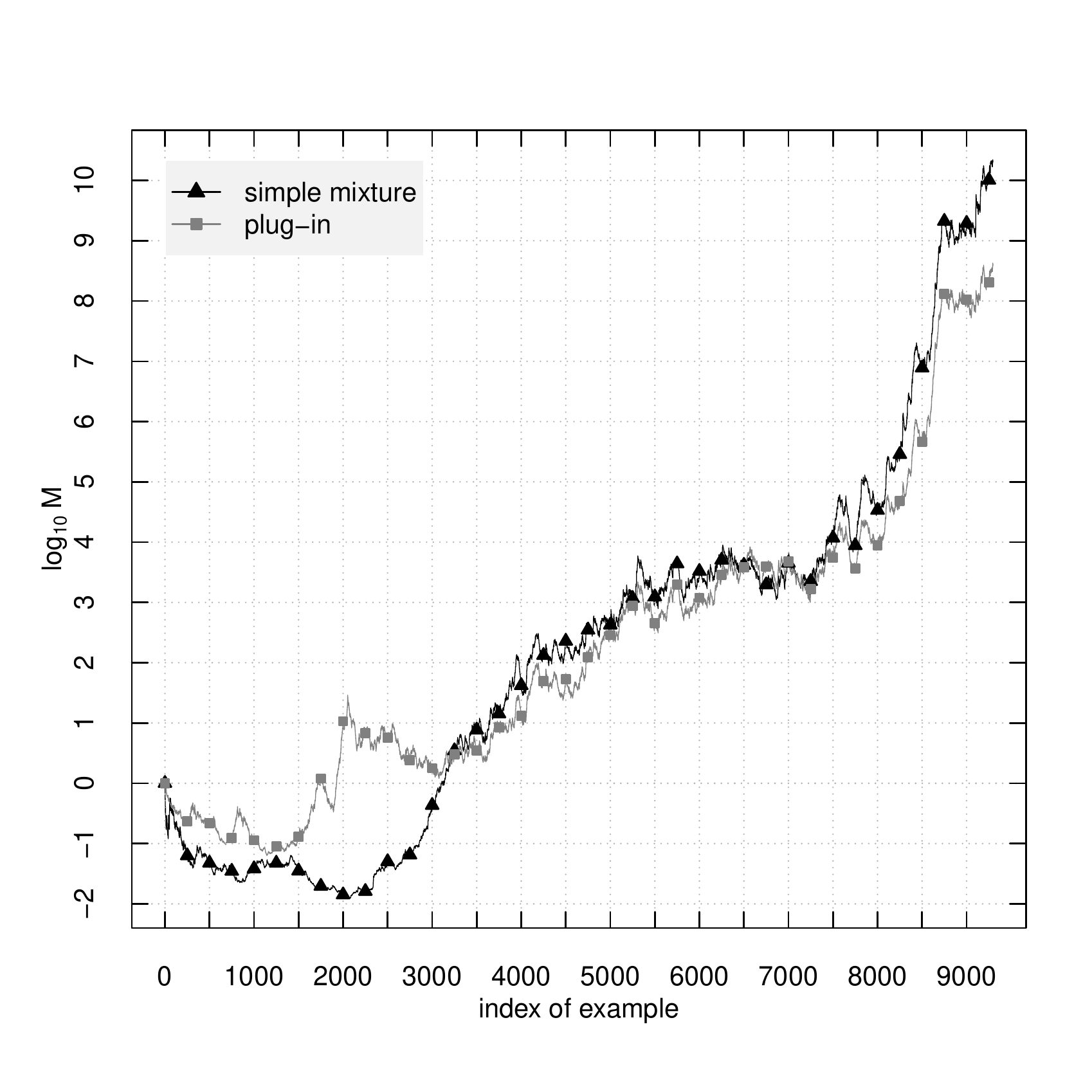}}
\caption{The growth of the martingales for the USPS dataset.
For the examples in the original order the exchangeability assumption is rejected:
the final martingale values are greater than $3.8\times10^{8}$.}
\label{fig:usps_martingales}
\end{center}
\vskip -0.2in
\end{figure}

\begin{figure}[tb]
\vskip 0.2in
\begin{center}
\centerline{\includegraphics[width=0.7\columnwidth]{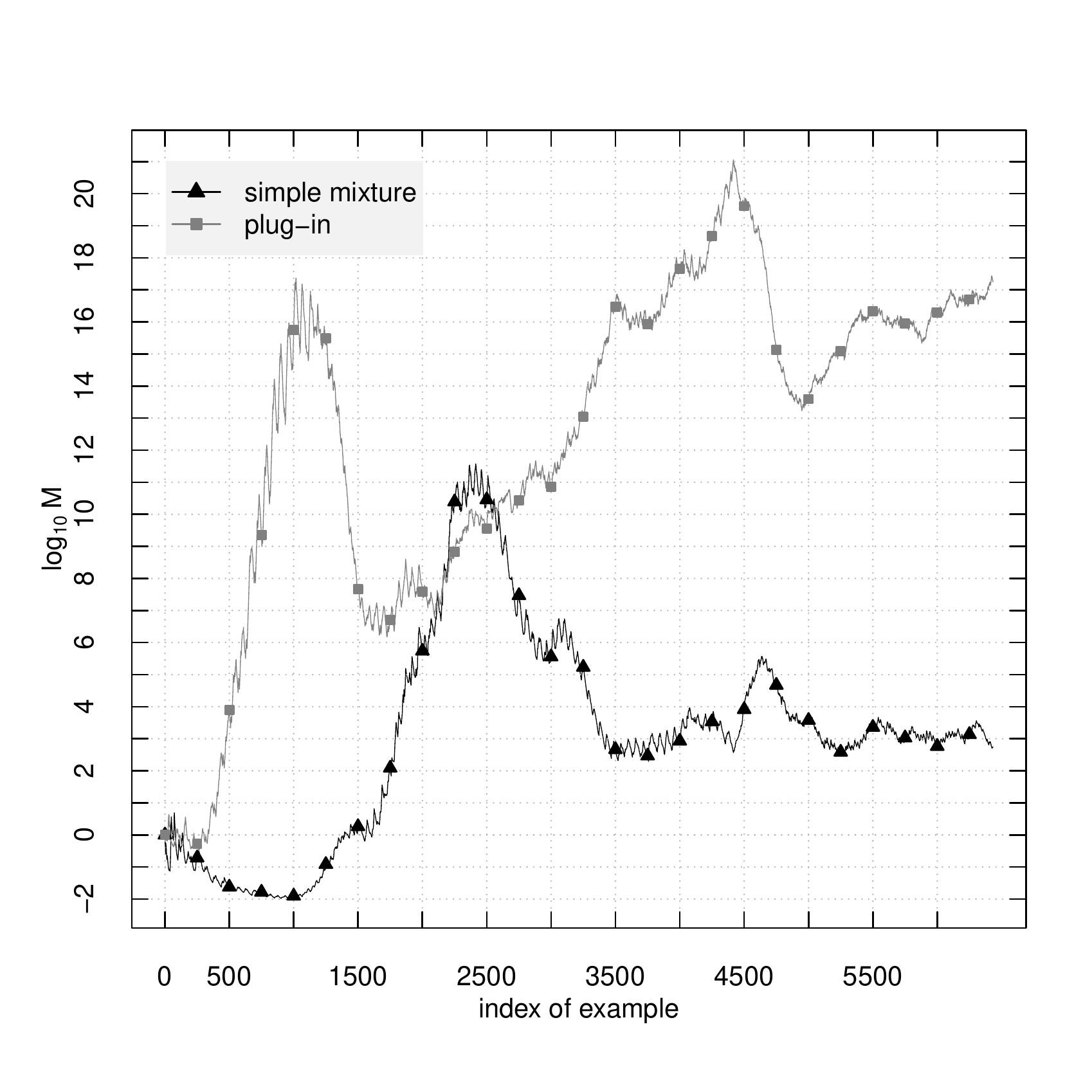}}
\caption{The growth of the martingales
for the Statlog Satellite dataset. For the examples in the original order the exchangeability
assumption is rejected: the final value of the simple mixture martingale is
$5.6\times10^{2}$, and the final value of the plug-in martingale is $1.8\times10^{17}$.}
\label{fig:statlog_martingales}
\end{center}
\vskip -0.2in
\end{figure}

\begin{figure}[tb]
\vskip 0.2in
\begin{center}
\centerline{\includegraphics[width=0.7\columnwidth]{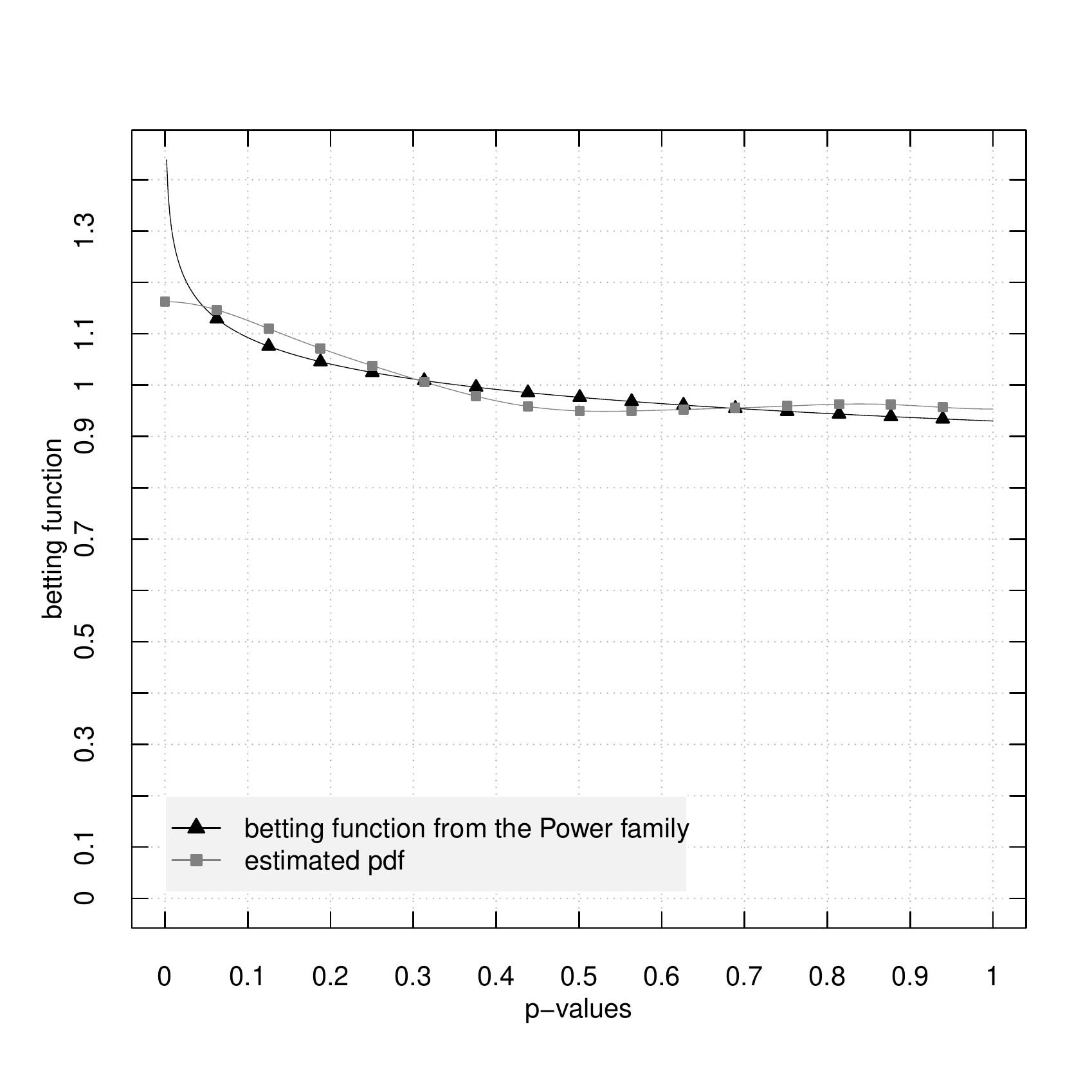}}
\caption{The betting functions for testing the USPS dataset for
examples in the original order.}
\label{fig:usps_betting_f}
\end{center}
\vskip -0.2in
\end{figure}

\begin{figure}[tb]
\vskip 0.2in
\begin{center}
\centerline{\includegraphics[width=0.7\columnwidth]{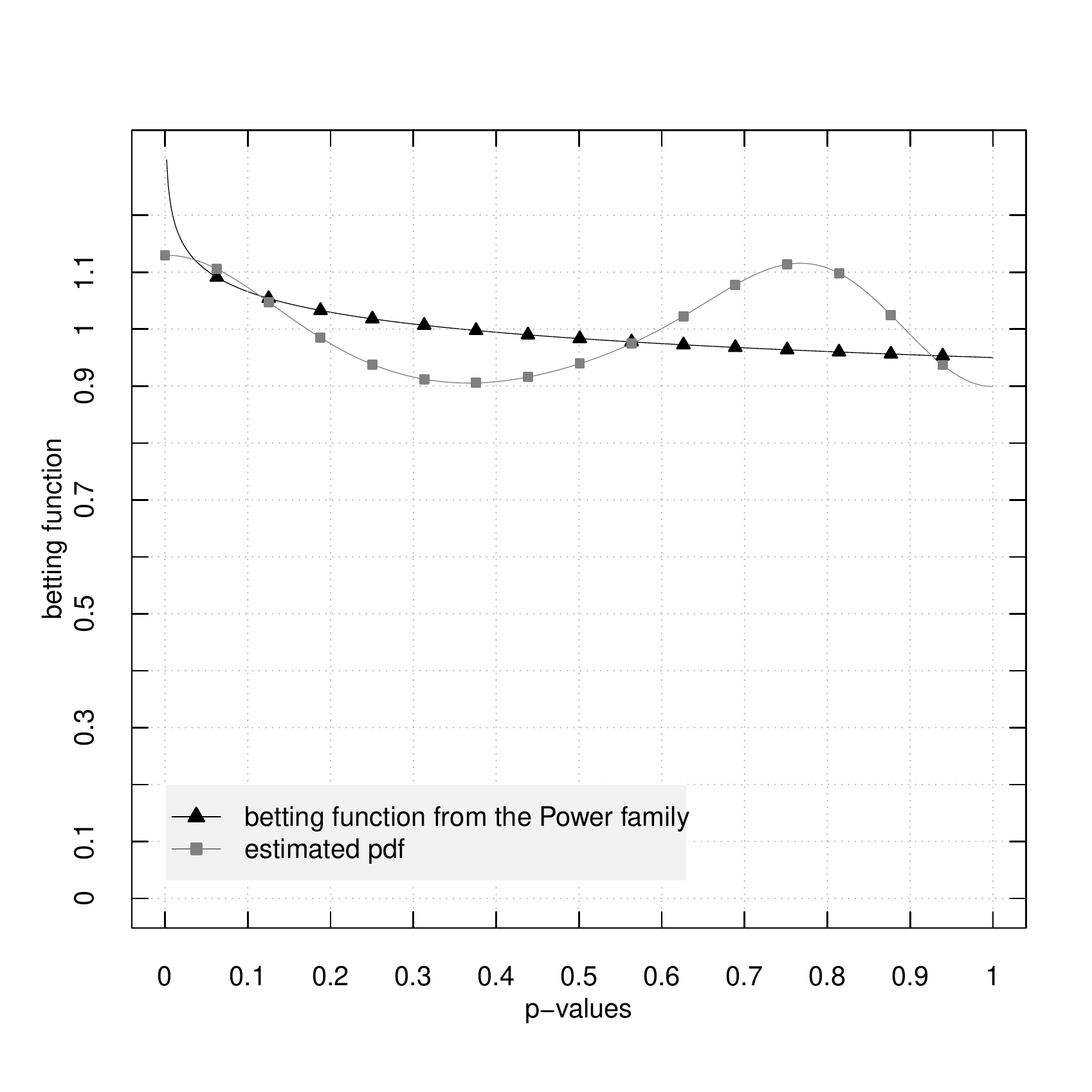}}
\caption{The betting functions for testing the Statlog Satellite dataset for
examples in the original order.}
\label{fig:statlog_betting_f}
\end{center}
\vskip -0.2in
\end{figure}
\subsection{USPS dataset}

\paragraph{Data}
The US Postal Service (USPS) dataset consists of $7291$ training examples and $2007$ test examples
of handwritten digits, from $0$ to $9$.
The data were collected from real-life zip codes. Each example
 is described by the $256$ attributes representing the pixels for displaying a digit on the $16\times16$
gray-scaled image and its label.
It is well known that the examples in this dataset are not perfectly exchangeable
\citep{vovk/nouretdinov/gammerman:2003},
and any reasonable test should reject exchangeability there.
In our experiments we merge the training and test sets and perform testing for the full dataset of $9298$
examples.

Figure \ref{fig:usps_mixed_martingales} shows the typical performance of the martingales
when the exchangeability assumption
is satisfied for sure: all examples have been randomly shuffled before the testing.

Figure \ref{fig:usps_martingales} shows the performance of the martingales when the examples arrive in the original order:
first 7291 of the training set and then 2007 of the test set.
The p-values are generated on-line by Algorithm \ref{alg:generate_pVal}
and the two martingales are calculated from the same sequence of p-values.
The final value for the simple mixture martingale is $2.0\times10^{10}$,
and the final value for the plug-in martingale is $3.9\times10^{8}$.

Figure \ref{fig:usps_betting_f} shows the betting functions
that correspond to the plug-in martingale and the ``best'' power martingale.
For the plug-in martingale, the function is the estimated probability density function
calculated using the whole sequence of p-values.
The betting function for the family of power martingale corresponds to the parameter $\varepsilon^\ast$
that provides the largest final value among all power martingales.
It gives a clue why we could not see advantages of the new approach for this dataset:
both martingales grew up to approximately the same level.
There is not much difference between the best betting functions for the old and new methods,
and the new method suffers because of its greater flexibility.

\subsection{Statlog Satellite dataset}

\paragraph{Data}
The Statlog Satellite dataset \citep{uci} consists of $6435$ satellite images
(divided into $4435$ training examples and $2000$ test examples).
The examples are $3\times3$ pixel sub-areas of the satellite picture, where each
pixel is described by four spectral values in different spectral bands.
Each example is represented by $36$ attributes and a label indicating the classification of the central pixel.
Labels are numbers from $1$ to $7$, excluding $6$.
The testing results are described below.

Figure \ref{fig:statlog_mixed_martingales} shows the performance of the  martingales
for randomly shuffled examples of the dataset.
As expected, the martingales do not reject the exchangeability assumption there.

Figure \ref{fig:statlog_martingales} presents the performance of the  martingales
when the examples arrive in the original order.
The final value for the simple mixture martingale is $5.6\times10^{2}$ and the final value for the plug-in martingale
is $1.8\times10^{17}$.
Again, the corresponding betting functions for the plug-in martingale and the ``best''
power martingale are presented in Figure~\ref{fig:statlog_betting_f}.
For this dataset the generated p-values have a tricky distribution.
The family of power betting functions $\varepsilon p^{\varepsilon - 1}$ cannot provide
a good approximation.
The power martingales lose on p-values close to the second peak of the p-values distribution.
But the plug-in martingale is more flexible and ends up with a much higher final value.

It can be argued that both methods, old and new, work for the Statlog Satellite dataset
in the sense of rejecting the exchangeability assumption
at any of the commonly used thresholds (such as 20 or 100).
However, the situation would have been different had the dataset consisted
of only the first 1000 examples:
the final value of the simple mixture martingale would have been $0.013$
whereas the final value of the plug-in martingale would have been $3.74 \times 10^{15}$.

\section{Discussion and conclusions}\label{sec:conclusion}

In this paper we have introduced a new way of constructing martingales
for testing exchangeability on-line.
We have shown that for stable sequences of p-values the new more adaptive martingale provides asymptotically
the best result compared with any other martingale with a fixed betting function.
The experiments of testing two real-life datasets have been presented.
Using the same sequence of p-values the plug-in martingale extracts approximately the same amount
or more information about the data-generating distribution
as compared to the previously introduced power martingales.

\begin{remark}
  The previous studies were based on the natural idea that lack of exchangeability
  leads to new examples looking strange as compared to the old ones
  and therefore to small p-values
  (for example, if the data-generating mechanism changes its regime
  and starts producing a different kind of examples).
  This is, however, a situation where lack of exchangeability makes the p-values cluster around 1:
  we observe examples that are ideal shapes of several kinds distorted by random noise,
  and the amount of noise decreases with time.
  Predicting the kind of a new example using the nonconformity measure (\ref{eq:NN})
  will then tend to produce large p-values.
\end{remark}



Our goal has been to find an exchangeability martingale
that does not need any assumptions about the p-values generated by the method of conformal prediction.
Our proposed martingale adapts to the unknown distribution of the p-values
by estimating a good betting function from the past data.
This is an example of the plug-in approach.
It is generally believed that the Bayesian approach is more efficient
than the plug-in approach \citep[see, e.g.,][p.~483]{bernardo/smith:2000}.
In our present context,
the Bayesian approach would involve choosing a prior distribution on the betting functions
and integrating the exchangeability martingales corresponding to these betting functions
over the prior distribution.
It is not clear yet whether this can be done efficiently and, if yes,
whether this can improve the performance of exchangeability martingales.

\section*{Acknowledgments}
\small{We are indebted to Royal Holloway, University of London,
for continued support and funding.
This work has also been supported by the EraSysBio+ grant SHIPREC
from the European Union, BBSRC and BMBF
and by the VLA grant
on machine learning algorithms.

We thank all reviewers for their valuable suggestions for improving the paper.

}

\ifCONF
  \bibliographystyle{icml2012}
\fi
\ifnotCONF
  \bibliographystyle{icml2012}
\fi
\bibliography{submission}
\end{document}